\documentclass{article}

\usepackage{PRIMEarxiv}

\usepackage{fancyhdr}
\usepackage[utf8]{inputenc} 
\usepackage[T1]{fontenc}    
\usepackage{hyperref}       
\usepackage{url}            
\usepackage{booktabs}       
\usepackage{amsfonts}       
\usepackage{nicefrac}       
\usepackage{microtype}      
\usepackage{xcolor}         

\usepackage{graphicx,subcaption,lipsum}
\usepackage{algorithm}
\usepackage{algpseudocode}
\usepackage{verbatim}       
\usepackage{amsthm}
\newtheorem{theorem}{Theorem}[section]

\newtheorem{corollary}{Corollary}[theorem]
\newtheorem{lemma}{Lemma}[section]
\title{An adaptive sampling algorithm for data-generation to build a data-manifold for physical problem surrogate modeling
}
\author{
  Chetra Mang, Axel TahmasebiMoradi,David Danan and Mouadh Yagoubi \\
  IRT SystemX \\
  2 Bd Thomas Gobert, 91120 Palaiseau, France \\
  \texttt{\{chetra.mang, a.tahmasebimoradi\}@irt-systemx.fr} \\
  \texttt{\{david.danan, mouadh.yagoubi\}@irt-systemx.fr} \\
}

\begin{document}
\maketitle

\begin{abstract}
Physical models classically involved Partial Differential equations (PDE) and depending of their underlying complexity and the level of accuracy required, and known to be computationally expensive to  numerically solve them. Thus, an idea would be to create a surrogate model relying on data generated by such solver. However, training such a model on an imbalanced data have been shown to be a very difficult task. Indeed, if the distribution of input leads to a poor response manifold representation, the model may not learn well and consequently, it may not predict the outcome with acceptable accuracy. In this work, we present an Adaptive Sampling Algorithm for Data Generation (ASADG) involving a physical model. As the initial input data may not accurately represent the response manifold in higher dimension, this algorithm iteratively adds input data into it. At each step the barycenter of each simplicial complex, that the manifold is discretized into, is added as new input data, if a certain threshold is satisfied. We demonstrate the efficiency of the data sampling algorithm in comparison with LHS method for generating more representative input data. To do so, we focus on the construction of a harmonic transport problem metamodel by generating data through a classical solver. By using such algorithm, it is possible to generate the same number of input data as LHS while providing a better representation of the response manifold.
\end{abstract}


\section{Introduction}
Machine Learning (ML) belongs to the broad family of Artificial Intelligence (AI) that gives the computers the capability to learn the hidden patterns in data and eventually make acceptable predictions. This is accomplished by optimizing a performance criterion with respect to data. A sub-branch of ML models involves surrogate models (metamodels); these models have been extensively used within the literature in the context of simulation optimization for real-world problems. It is especially useful to address the situation where one would like to replace classical simulation tools, that are computationally expensive, by a more cost-effective alternative. The learning capability of surrogate models models depend on different factors, one being the quality of the data. In practice, how exactly to gather/generate data in certain application domains so that the models trained on them are generally egalitarian, is a challenging task. One of the earliest method for data generation is "equal sampling" (equal space  or probability sampling); using a classical Latin-Hypercube Sampling (LHS) method, every dimension is parametric space is divided (almost) equally. Other equal sampling methods are Cartesian sampler (equal space) and uniform sampler (equal probability).

In general, there are two types of conventional sampling approaches; 1- fixed approaches and 2- adaptive approaches. For the latter, a criterion must be defined to evaluate the model. The former is also categorized into 1- deterministic, and 2- stochastic methods. Either type can be further branched into 1- one-shot, and 2- sequential techniques. The commonly used one-shot approaches are: Factorial (deterministic), Latin Hypercube Sampling (stochastic); and the conventional sequential approaches are: Sobol/Halton/Faure sequences (deterministic) and Monte Carlo (stochastic) \cite{dupuis2019surrogate}. 

With the use of fixed sampling approaches, one can generate a parametric spaces within a reasonable time frame. However, this doesn't necessarily mean that the associated points on the response manifolds are also well positioned from each other. Moreover, the associated points on the response manifold may not cover the area with drastic changes. To alleviate these problems with regard to equal sampling, one may increase the number of samples. Unfortunately, this may not be the best idea, as for some problems, generating a pair of data sample (input and output data) can be computationally costly. This is usually the case for dataset created by solving potentially nonlinear PDE defined on large domains, e.g. Computational Fluid Dynamic (CFD) using classical discretization approaches such as Finite Element, Finite Difference, Finite Volume.
As metamodels have gained intensive attention, various techniques have been continuously developed in combination with suitable sampling strategies. It includes, while not exhaustive, non-intrusive POD \cite{guenot2013adaptive}, Support Vector Machine \cite{mandelli2012adaptive} or Regression \cite{zhou2015adaptive}, and Kriging \cite{fuhg2021state}. For more information regarding the metamodels, we refer to the survey \cite{liu2018survey} and the references therein. Note that the sampling strategies described in this context are, in general, metamodel agnostic.

By nature, the adaptive sampling approaches are sequential. Using these methods allows us to choose more "efficient" points in parametric space and as a result, the metamodel is likely to give less error with fewer samples. 
As stated before, the adaptive approaches require a well-defined criterion to search for new points in parametric space. Generally, there are two types of search in this space, i.e. explorative, and exploitative searches. Explorative search is about finding the regions (space-filling design) in parametric space that are arid of data points. Conversely, exploitative methods try to spot the regions whose associated mappings on response manifold have considerable gradient. Consequently, by adding new data points, the metamodel's performance is evaluated, and further addition of points is decided. 

\subsection{Related work}
To improve the efficiency and accuracy of POD data-fitting surrogate models, it was suggested in \cite{guenot2013adaptive} to either improve the basis or the coefficient model. To improve the modal basis, a parameter to show the influence of each snapshot on the modal basis was computed. On the other hand, to improve the modal coefficients, two approaches, 1- on at a time (when there are significant differences in the quality of the modal coefficients), and 2- all together (when the quality of these coefficients are similar), were proposed.   

Research efforts focus on a compromise between with two aspects: generate sufficient number of samples for the surrogate model to be efficient and minimize that number because of the cost to generate them. Thus, one natural choice to address those concerns is adaptive sampling. Instead of generating a massive, and probably redundant, amount of data, a trend is to provide a more moderate amount of representative data for the underlying problem at hand. Indeed, as several authors observed \cite{liu2018survey,ajdari2014adaptive}, pure one-shot experimental design suffers from its inability to capture special characteristics of the surface response manifold. Besides, the sufficient number and suitable distribution of design points required to do so are not usually known beforehand. To circumvent these limitations, sequential experimental design is used to explore the input space more efficiently and add iteratively sample points that are likely to improve the surface response manifold. Note that a one shot procedure is still used to initialize the sequential experimental design. 

The two search strategies classically employed are explorative and exploitative, as mentioned in the introduction, however several authors \cite{liu2018survey,ajdari2014adaptive,jiang2018adaptive} propose to combine both. The underlying idea is to find an acceptable trade-off between the conflicting local exploitation and global exploration, expressed as a scoring function based on the combination of both terms. Indeed, at a given step, the search algorithm aims to find a new point 1-whose close vicinity is not part of the samples already available and 2-that contains significant information about the surface’s characteristics to improve the underlying problem representativity. However, in practice, it remains difficult to find a point that meet simultaneously these two conditions, thus the aggregation procedure.

Interestingly enough, both \cite{ajdari2014adaptive,jiang2018adaptive} rely on an sampling algorithm based on a Delaunay triangulation for the decomposition of the search space initialized by a discrete search space. The former evaluates the scores of each triangle as follows: the area computation for the exploration, the total variation of response between nodes for the exploitation and a linear combination of both for the total score. They also propose several enhancement to avoid clustering of points in a region. The latter proceed as follows for each triangle: compute the area, compute the prediction error of the metamodel at the centroid and finally calculate the weights based on entropy weight method for the global score.
Regarding the high-dimensional adaptive sampling, for which the required number of points grow exponentially with dimensions, the survey \cite{liu2018survey} provides several insights. Nevertheless, because of the “curse of dimensionality” and the inability to leads properly the adaptive algorithm because of the lack of samples, it remains an intractable task according to the authors.

As for triangulation technique on higher dimension manifold, one would encounter a fundamental question whether any topological manifold is triangulable. Lashof et al. and Whitehead~\cite{lashof1969triangulation,whitehead1940c1} has partially answer this question. Every smooth manifold is triangulable. Though, we can reach a situation which the data manifold is not smooth. Moise and Radó~\cite{moise1952affine,rado1925begriff} have proven that every n-manifold such that $n<4$ is smooth, and Freedman and Kirby et al.~\cite{freedman1982topology,kirby1969triangulation} have shown that  for $n\geq 4$, there are some manifolds which are not triangulable. However, Boissonnat et al.~\cite{boissonnat2018delaunay} proposed perturbation data point algorithm that can be used to triangulate the compact Riemannien manifold and can extend to the higher dimension manifold which may tackle this problem.

Regarding projector operators, there are many methods to construct one in order to reduce the high dimensional input data manifold. The well known one is t-SNE method which is based on a variation stochastic embedded \cite{van2008visualizing}. This method can efficiently represent a single map that depicts data structure at different scales. It is known to be used for high dimension data visualization. However, this method is less efficient to inverse data point from the embedded lower dimension space to the original one. Another well known reduced component method is locally linear embedded (LLE). This method is based the non-linear spectral dimensionality reduction \cite{roweis2000nonlinear}. It focuses on preserving the local data structure in the embedded space. The method relies on proximity of the data points in the data manifold. It searches to do local fitting, hence for a far away data points, it will become farther.

\subsection{Main contribution}
This work proposes a new adaptive sampling algorithm based on physical models. Based on solutions of a physical solver corresponding to each point, and by use of triangulation technique, simplicial complex (depending on the number of dimensions in parametric space) are created. New points are introduced on the barycenters of the created elements and then they are added to the initial input data if a threshold is satisfied. The proposed algorithm uses a technique that allows to carry out the triangulation for higher dimensions. In addition, while the method described in this work is applied to the construction of a harmonic transport metamodel, its core principles aim to be generic and could be applied to any problems.

The remainder of this paper is organized as follows. In section II, a new methodology for adaptive sampling is introduced. Next, this adaptive approach is applied to generate data points for a harmonic transport equation in section III. In what follows, the comparison between LHS and the proposed method on data input distribution and training of fully connected neural network's performance are carried out, the results are shown and the performance of the proposed adaptive sampling approach is discussed. Finally, the conclusion remarks are presented.

\section{Methodology}
\subsection{Adaptive sampling algorithm}
We start by considering a physical model described by an implicite function or a partial differential equation:
\begin{equation}
F(s,y(x))=0
\label{eq_pde}
\end{equation}

where
\begin{itemize}
\item $ s \in \mathcal S \subset \mathbf R^N$ where $\mathcal S$ is a set of the equation parameters as an input dataset with dimension $N$;
\item $y(x) \in \mathcal Y \subset \mathbf R^M$ where $\mathcal Y$ is a set of the solution to the equation (\ref{eq_pde}) with dimension $M$;
\item $x \in \Omega$ where $\Omega$ is a spatial domain. 
\end{itemize}

The equation (\ref{eq_pde}) can be solved using discretized methods. Without the abuse of notation, we denote its discretized solution as $y \in \mathcal Y \subset \mathbf R^M$ where $\mathcal Y$ is a set of discretized solution as an output dataset with dimension $M$ corresponding to the input dataset $\mathcal S$. Therefore, the equation (\ref{eq_pde}) can be written as: 

\begin{equation}
F(s,y) \leq \epsilon\ \textrm{for}\ (s,y)\in \mathcal S\times\mathcal Y
\label{eq_pde_d}
\end{equation}

Where $\epsilon$ is a tolerance constant set for solving the equation (\ref{eq_pde}).

We define a set of Euclidean coordinate point $\mathcal P = \mathcal S \times \mathcal Y $ and a function $\chi(x) = 0\ \textrm{if}\  x\leq0\ \textrm{and}\ \chi(x)=1\ \textrm{else}$. Therefore, the response manifold corresponding to the dataset $(X,Y)$ can be written as follows:

\begin{equation}
\mathcal{M}=\{P\in \mathcal{P} | \chi (F(P) -  \epsilon) = 0\}
\label{eq_manifold}
\end{equation}

The proposed method is based on the triangulation of the response manifold (\ref{eq_manifold}) into a set of simplicial complex. Then, we build a barycenter set $\mathbf P_G$ corresponding to the simplicial complex set .

Even though $\mathcal P_G$ 

is a set of the manifold barycenters, we cannot guarantee that $ \mathcal P_G \subset \mathcal{M}$. Indeed, if $P_G \in \mathcal{M} \cap  \mathcal P_G $, it means that output is a simple linear combination of the solution at the nodes of the simplicial complex while being also a solution of equation (\ref{eq_pde}) however there is no reason it should be the case.

We create a manifold point $\tilde{P}_G \in \tilde{\mathcal{P}}_G$ where $\tilde{\mathcal{P}}_G$ is the set of the solution of the equation (\ref{eq_pde_d}) corresponding to the barycenter set so that $\tilde{P}_G \in \mathcal{M}$.

Next, we define a condition whether one needs to add the solution point $\tilde P_G$ at the barycenter to the original manifold. To do so, we define a metric as follows:
\begin{equation}
\rho_G=F(P_G)
\label{eq_metric}
\end{equation}

With respect to the defined metric $\rho_G$, we can assess whether it is relevant to add a new point to the manifold by introducing  a threshold $\rho_{\epsilon}$ at the barycenter coordinate position. 

The drawback of this criteria is that if we set $\rho_{\epsilon}$ very small, many points are likely to be added to the manifold. If we let that happen, the number of points will be increasing exponentially. Hence, we have to find a way to control their number. We propose a relaxation threshold which allows to first add some points with the largest precision error. This threshold is defined as follows:

\begin{equation}
\rho_{\epsilon,k} = \frac{\rho_{\epsilon,0}}{\lambda^k}\label{eq_relaxation}
\end{equation}

Where
\begin{itemize}
\item $k$ is the iteration number for manifold point generation ;
\item $\rho_{\epsilon,0}$ is the initial threshold;
\item $\lambda > 1$ is a decayed factor. 
\end{itemize}

We nearly have all the ingredients for the method. The only step remaining is to define how to start the algorithm. To do so, the initial input data set equation parameters is obtained by the corner coordinate points of the bounding box's input data set given by the cartesian product of the minimal/maximal value of each parameter. For instance, let $p_1$ and $p_2$ be two equation parameters, the input dataset contains four points: 
\begin{equation}
\begin{array}{ll}
P_{Initial}=&\{(min(p_1),min(p_2)),\\
&(min(p_1),max(p_2)),\\
&(max(p_1),min(p_2)),\\
&(max(p_1),max(p_2))\}.
\end{array}
\end{equation}
It is a one-shot sampling that generates the minimum necessary data set to avoid creating non-necessary points by parametric space exploration, because of the high cost of simulation. Then, the corresponding discretized solutions is simulated to create a manifold's initial points.

The main purpose of the developed algorithm is to generate a small input data for which the precision is satisfied. However, to reach a very good precision, large number of the solution at barycenters or large amount of computational resources may be needed. To overcome this, we can enrich the stopping criteria by setting the time or sample number limit.

As mentioned in the state of the art that for manifold's dimension lower than 4, any topological manifold is triangulable. Therefore, We focus the scope of this work on the surface manifold in the rest of the paper.

\subsection{Reduced surface manifold}\label{sec:reducedmanifold}

From the original manifold,  we may reduce the dimension of the manifold coordinate's point based on the reduced dimension of the input and output data. This is called the reduced order manifold.

We detail in the following paragraph how to construct the projector of the input data $\mu: \mathbf{R}^N \rightarrow \mathbf{R}^2$ and $p\rightarrow \bar{p}$ and its pseudo inverse which is defined as $\mu^{-1}:\mathbf{R}^2 \rightarrow \mathbf{R}^N$ and $||\mu^{-1}(\bar{p}) - p||_{L^2}<\epsilon$ for sufficiently small $\epsilon$. This construction allows to proof the lemma \ref{lemma:proj} for the theorem \ref{thm:reducedmanifold}.

It should be noted that, in the case of $N=2$, the projector $\mu$ is an identity. It is the same for the case $N=1$, but it is not included in this study.

A recent developed method called Uniform Manifold Approximation Projection (UMAP) has been shown to be more efficient than t-SNE method according to \cite{leland2018uniform}. We use UMAP as the projector $\mu$. Although UMAP provides the inverse transform of the lower dimension data of the embedded space to the original data, it creates some bias in the original data. Even with very small global error ($||\mu^{-1}(\bar{p}) - p||_{L_2}$), the local error ($|\mu^{-1}(\bar{p})_i - p_i|, \mu^{-1}(\bar{p})_i\in \mu^{-1}(\bar{p}), p_i \in p$) can be large.
Therefore, it can generate the input data outside the original input data domain which is unacceptable for the physical simulation. Hence, we propose the inversion of $\mu$ based on the preprocessing data grid. The main idea is to sample a data grid using uniform or Cartesian sampler to generate a refine input data. Then, we apply UMAP on the whole input data grid to reduce the original high dimension data to two dimension one. Thus, for each two dimensional data point on the grid, there is a corresponding original dimensional data point. If, for example, a new two dimensional point is employed, thus we need the corresponding original data for the simulation. The inverse projection $\mu^{-1}$ injects the new two dimensional data point to the closest two dimensional data point on the grid, hence we can obtain the corresponding original data point on the grid for the simulation. Doing so, we guarantee the input data is always in the input data domain for consistent physical simulation. However, global error ($||\mu^{-1}(\bar{p})-p||_{L^2}$) can be generated on the position of the coordinate point. This error has an upper bound depending on the size of the grid $h$.\textbf{When $h$ is sufficiently small, then $||\mu^{-1}(\bar{p})-p||_{L^2}<\epsilon$}.

\begin{lemma}\label{lemma:proj}
For a high dimension continuous manifold $\mathcal{M}$ defined in (\ref{eq_manifold}) , there exists a projector $\pi:\mathbf{R}^N\times\mathbf{R}^M\rightarrow \mathbf{R}^2\times\mathbf{R}$ and $P\rightarrow \bar{P}$ for $P\in \mathcal{M}$ and $\bar{P} = \pi(P)$ such that its pseudo inverse $\pi^{-1}:\mathbf{R}^2\times\mathbf{R}\longrightarrow \mathbf{R}^N\times\mathbf{R}^M$ and $||\pi^{-1}(\bar{P})-P||_{L_2}<\epsilon$ for sufficiently small $\epsilon$.
\label{lm_pi}
\end{lemma}

\begin{theorem}\label{thm:reducedmanifold}
For a high dimension continuous manifold $\mathcal{M}$ defined in (\ref{eq_manifold}) , there exists a corresponding reduced surface manifold $\bar{\mathcal{M}}$.
\label{th_manifold_reduced}
\end{theorem}

\begin{corollary}\label{cor:reducedmanifold}
For a high dimension continuous manifold defined in (\ref{eq_manifold}) , there exists a corresponding reduced order manifold for any lower dimension.
\label{th_manifold_reduced_general}
\end{corollary}

The proofs of the lemma \ref{lemma:proj} and theorem \ref{thm:reducedmanifold} are found in Appendix \ref{sec:theorem}.

\section{Results and discussions}
\subsection{Harmonic transport problem}\label{subsec:harmtran}

The Goldstein equation~\cite{bensalah2022mathematical} is used to model the sound propagation. A simplification of this model is to solve a 1D harmonic transport problem with a high variation source function.

The harmonic transport equation is described by the following partial differential equation:
\begin{equation}
F(p,y(x)) = m\frac{\partial y}{\partial x} - i ky(x) - g(x) = 0
\label{eq_gsd}
\end{equation}
where \\
\begin{itemize}
\item $i = \sqrt{-1}$ ;
\item $m$ is mach number which describes the velocity of the propagate wave function;
\item $k$ is wave number;
\item $g$ is the source function which describes the external force pushing on the wave medium;
\item $x = [0,1]$  is spatial domain;
\item $y(x=0)=0$ is the boundary condition.
\end{itemize}
The discretized solution of the equation (\ref{eq_gsd}) can be written as $y_R = (Re(y),Im(y)) \in \mathbf{R}^{M}$. Two cases of the input data are carried out in this study. 

For the first case, we focus on the low dimensional input data where the source function is presented by a parametric function $g(x) = a e^{\alpha xi} + e^{(\alpha xi - \frac{(x-x_m)^2}{2\sigma^2})}$. Thus, the parameters of the equation (\ref{eq_gsd}) is $(m,x_m,k,a,\alpha,\sigma)$. We only vary two parameters $m$ and $x_m$ and fix the others. So, the input and output data are $s=(m,x_m) \in \mathbf{R}^2$ and $y_R\in \mathbf{R}^{M}$ respectively. 

For the second case, we work on the high dimensional input data where $m=m(x)$ and $g_R=(Re(g),Im(g))$ are represented by a parametric Chebeshev polynomial that vary at each discretized spatial position and the parameter $k$ is fixed. So, the input and output data are $s=(m,g_R)\in \mathbf{R}^N, N=\frac{3}{2}M$ and $y_R\in \mathbf{R}^{M}$ respectively.

\subsection{Influences of ASADG's parameters} 

This section illustrates the influence of the iteration number and initial threshold on ASADG's outcomes. We consider the low dimensional input data where the source function is presented by a parametric function as in the first case described in the section \ref{subsec:harmtran}.

Fig.\@ \ref{fig_manifold} shows the evolution of the data manifold for different algorithm's iterations. At the initial step, the initial points of the input data are determined. The surface manifold consists of four initial points which are triangulated into two triangles (Fig.\@ \ref{fig_manifold}a). Next step, the barycenters of the two triangle are computed. It is depicted in Fig.\@ \ref{fig_manifold}b and c as red dot points and the corresponding solution at the barycenters are generated through physical simulations. These points shown as the blue dot points are added to the manifold. The generated points and its corresponding initial barycenters are inside the red dashed squares. Again, the manifold triangulation are performed. This procedure are conducted until the stopping criteria occurred. Illustrated in Fig.\@ \ref{fig_manifold}d, the isolated red points display the previous step barycenters. Their metric is larger than the relaxation threshold. Therefore, these points are eliminated before the triangulation preparation for the next step point generation. Fig.\@ \ref{fig_manifold_sample}a shows the 3000 point triangulated manifold. Their distribution is highly concentrated where the manifold curvature is large. This distribution can lead to an approximated representation of data manifold with this sampling number as comparing to the LHS sampling of 50000 points illustrated in Fig. \@ \ref{fig_manifold_sample}b. The ASADG manifold points can capture the wavelet surface at the high curvature part of the manifold as shown in the dashed red box at the right side of Fig. \@ \ref{fig_manifold_sample}a. Moreover, at the lower curvature part of the manifold, the manifold feature is also be able to capture the wavelet form but less representative than that in the high curvature zone. 

Fig.\@ \ref{fig_manifold_param}a displays evolution's of the number of sampling points (red curve) and the mean of metric value (blue curve) according to the generation step. We observe that the number of sampling points is slowly increasing thank to the decay factor and the initial threshold; though at iteration number 20, it starts exponentially growing. On the contrary, the metric mean value are rapidly decreasing. Even though the metric has not yet reach its threshold, the mean metric value seems to be stable at iteration number 20 while the number of sampling points continues increasing. The metric stability can be a criteria to determine the data manifold representativity. 

The initial threshold for a given decay factor ($\lambda=1.5$) plays an important role in reducing the number of added points at some generation iteration. In this case, the number of iteration is of $20$, and the upper limit of the total generated points is of $10000$. Fig.\@ \ref{fig_manifold_param}b illustrates the number of sampling points according to the initial threshold. When the initial threshold is $10^{-2}$ more than $10000$ points are added to the manifold at the end of generation iteration. This can cause a heavy computational cost involving physical simulation. Moreover, the number of the barycenters at the next step is possibly exponentially growing.

When the initial threshold is large, less generated points are added to the manifold. For instance, if the initial threshold is of $10^3$ 
, there is only $1116$ total generated points added to the manifold at the end of generation iteration. It'd better to have a sufficiently large initial threshold so that the number of generated point rate's are slowly increased. 

\begin{figure}[ht]
\centering
\begin{subfigure}{0.5\textwidth}
  \centering
\includegraphics[height=4.cm]{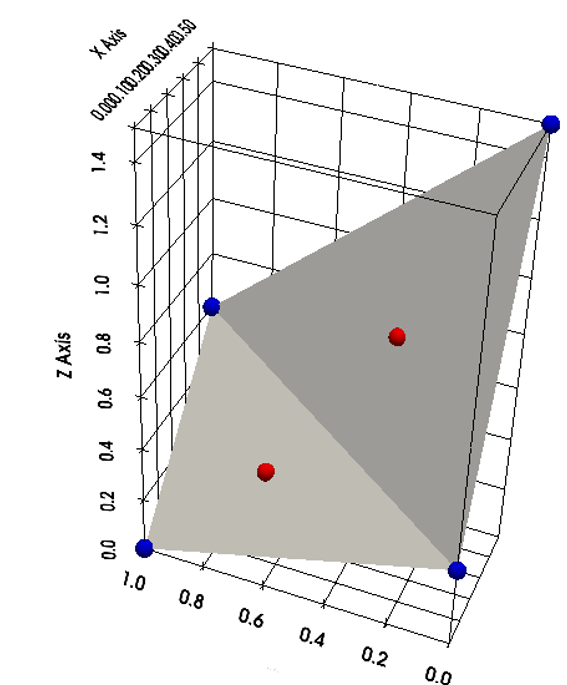}
  \caption{}
  \label{fig:sub1_fig_manifold}
\end{subfigure}%
\begin{subfigure}{0.5\textwidth}
  \centering
\includegraphics[height=4.cm]{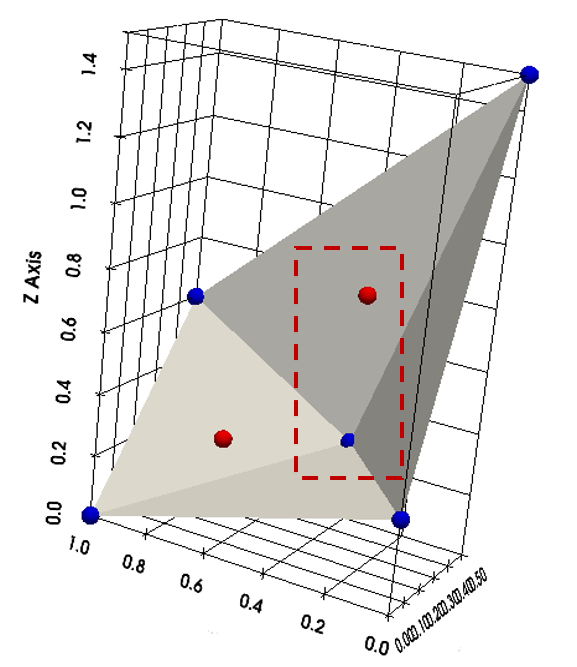}
  \caption{}
  \label{fig:sub2_fig_manifold}
\end{subfigure}%
\vfill
\begin{subfigure}{0.5\textwidth}
  \centering
\includegraphics[height=3.5cm]{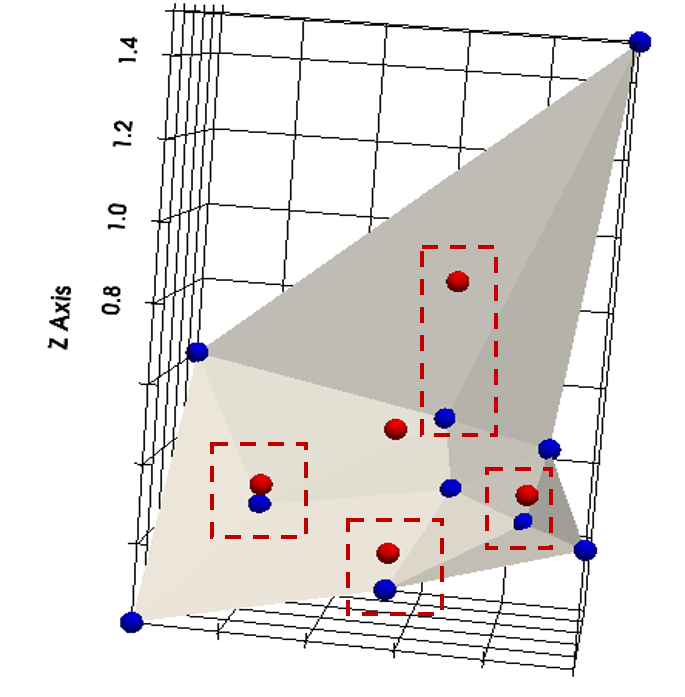}
  \caption{}
  \label{fig:sub3_fig_manifold}
\end{subfigure}%
\begin{subfigure}{0.5\textwidth}
  \centering
\includegraphics[height=4.2cm]{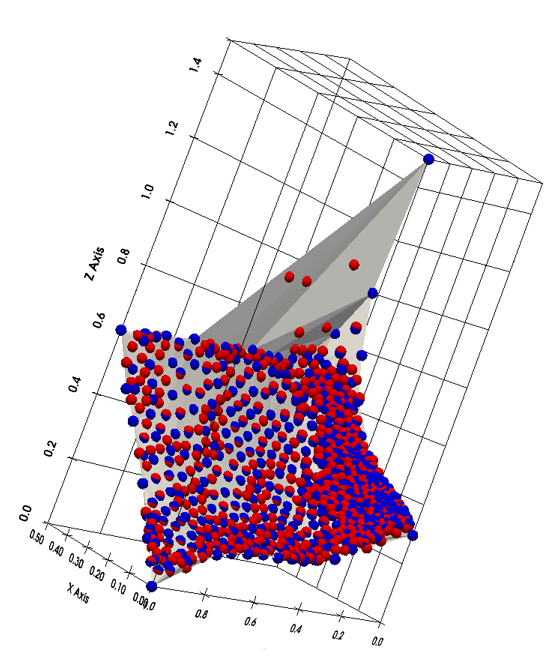}
  \caption{}
  \label{fig:sub4_fig_manifold}
\end{subfigure}%
\caption{Evolution of surface manifold according iteration number: (a) initialization, (b) 10 iterations, (c) 12 iterations and (d) 17 iterations.} \label{fig_manifold}
\end{figure}

\begin{figure}[ht]
\begin{subfigure}{0.5\textwidth}
  \centering
\includegraphics[height=4.5cm]{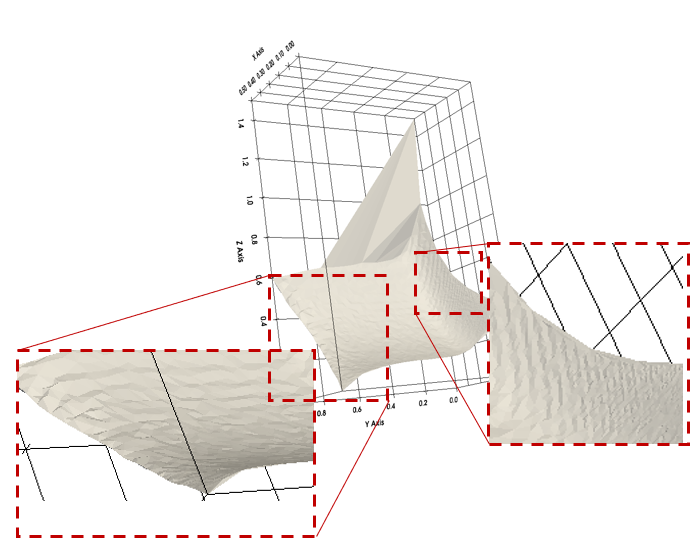}
  \caption{}
  \label{fig:sub5_fig_manifold}
\end{subfigure}%
\begin{subfigure}{0.5\textwidth}
  \centering
\includegraphics[height=4.5cm]{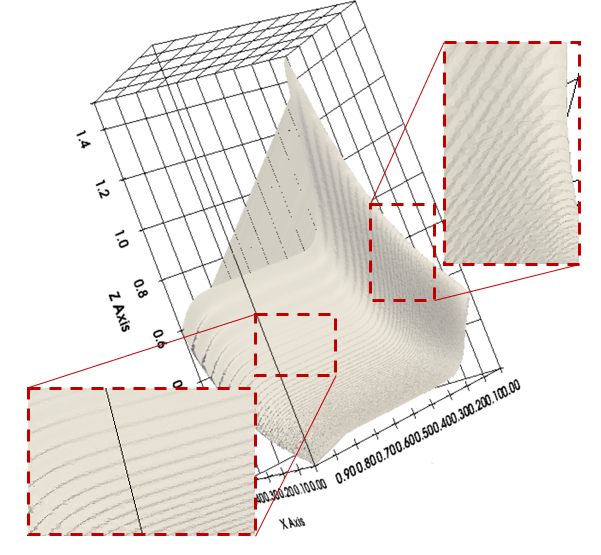}
  \caption{}
  \label{fig:sub6_fig_manifold}
\end{subfigure}%
\caption{Evolution of surface manifold according iteration number: (a) ASADG with 3000 points and (b) LHS with 50000 points.} \label{fig_manifold_sample}
\end{figure}

\begin{figure}[ht]
\centering
\begin{subfigure}{0.5\textwidth}
  \centering
\includegraphics[height=4.cm]{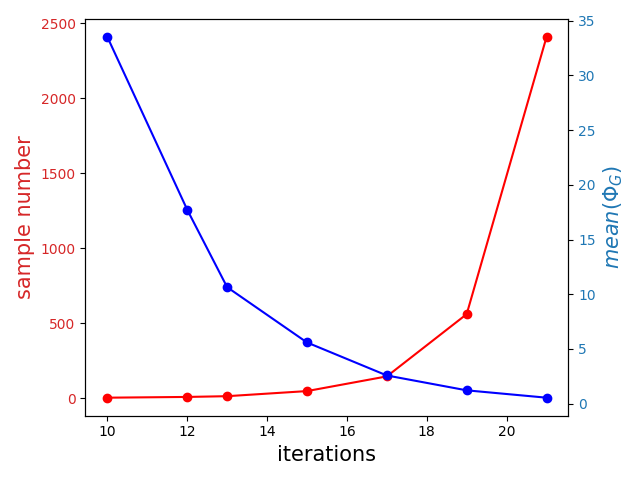}
  \caption{}
  \label{fig:sub1_fig_manifold_param}
\end{subfigure}%
\begin{subfigure}{0.5\textwidth}
  \centering
\includegraphics[height=4.cm]{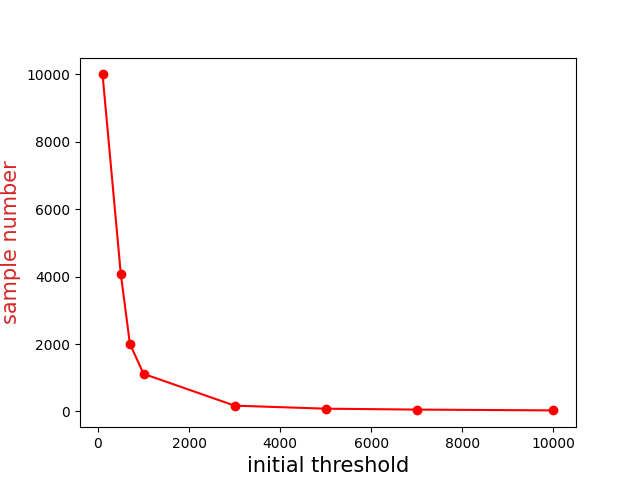}
  \caption{}
  \label{fig:sub2_fig_manifold_param}
\end{subfigure}%
\caption{Evolution of  (a) the sampling number (red curve) and the metric mean value at the barycenters (blue curve) according to each generation step and (b) the sampling number according to relaxation constant for a fixed number of iteration.} \label{fig_manifold_param}
\end{figure}

\subsection{Comparison of different sampling methods}\label{subsec:compsamplers}

This section reports the comparison between ASADG sampler and LHS sampler for the low dimensional data and uniform sampler for the high dimensional data.
\subsubsection{Low dimensional input data}
We select an ASADG sampler from the previous section with number of sampling points of 500. On the other hand, the LHS sampler is applied for the same input data domain with the same number of sampling points.

Fig.\@ \ref{fig_method_dim_low}a displays the distribution of the sample points in the input data domain using ASADG (red points) and LHS (blue points). Sample points using LHS are well distributed over the input data domain. On the contrary, those using ASADG are centred at the lower-right part, and sparse elsewhere. It is due to the well known property of LHS which characterize the optimal distribution of the sampler points without any knowledge of data manifold, on the other hand, the distribution of the sample points by ASADG respects the data manifold feature. 

\begin{figure}[ht]
\centering
\begin{subfigure}{0.5\textwidth}
  \centering
\includegraphics[height=4cm]{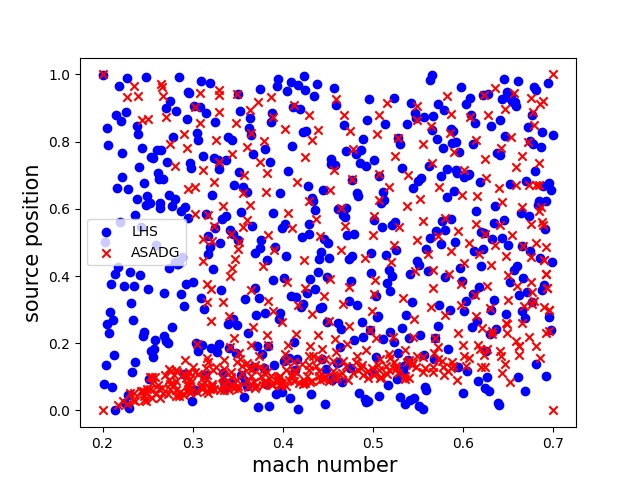}
  \caption{}
  \label{fig:sub1_fig_low_dim}
\end{subfigure}%
\begin{subfigure}{0.5\textwidth}
  \centering
\includegraphics[height=5cm]{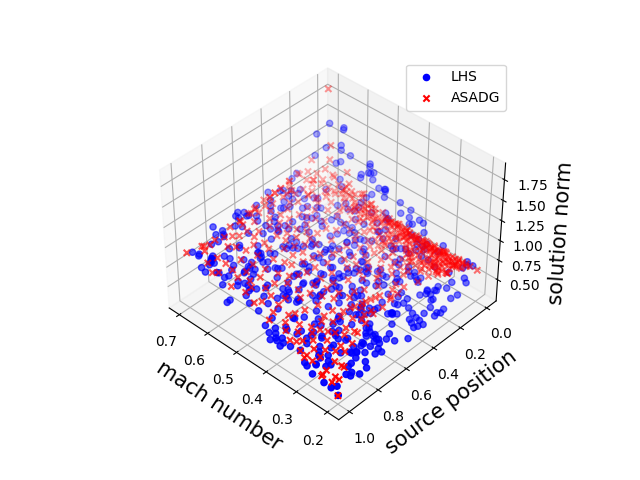}
  \caption{}
  \label{fig:sub2_fig_low_dim}
\end{subfigure}%
\caption{Distribution of sample points by ASADG (red points) and LHS (blue points) of (a) input data and (b) manifold points in the input data domain. } \label{fig_method_dim_low}
\end{figure}

\subsubsection{High dimensional input data}
We need to construct the reduced surface manifold before applying ASADG sampler. LHS, in this case, cannot be applied because of computational cost (3072 data input dimension), instead, we apply uniform sampler to sample the parametric Chebyshev polynomial for input data set. Each sampler generates 500 sample points in the same input data domain.

To visualize the data manifold, we apply the projector $\pi$ on data samples by ASADG and uniform samplers. Fig.\@ \ref{fig_method_dim_high}a illustrates reduced coordinates of sample points in the input data domain by ASADG (blue points) and uniform sampler (red points). With uniform sampler, the distribution of sample reduced coordinates are concentrated, in contrary with ASADG, are largely spatially distributed. 
Because of the chosen projector properties, more sample points are dispersed in parametric space, more different features of data input are presented. Fig.\@ \ref{fig_method_dim_high}b shows reduced surface manifold sample points by ASADG (red points) and uniform (blue points) samplers. The sample points of the obtained via ASADG are extended in a large parametric space, on the other hand, by those of the uniform sampler are concentrated in a small parametric space.

\begin{figure}[ht]
\begin{subfigure}{0.5\textwidth}
  \centering
\includegraphics[height=4cm]{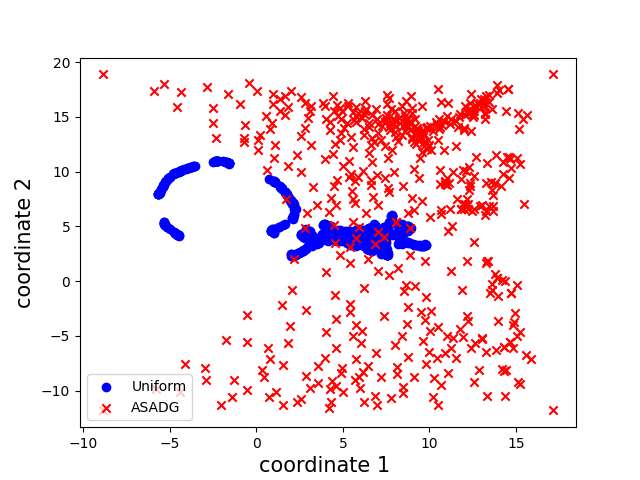}
  \caption{}
  \label{fig:sub1_fig_high_dim}
\end{subfigure}%
\begin{subfigure}{0.5\textwidth}
  \centering
\includegraphics[height=5cm]{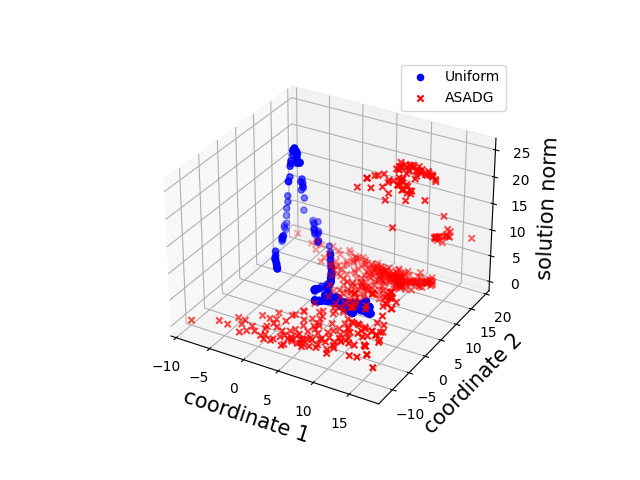}
  \caption{}
  \label{fig:sub2_fig_high_dim}
\end{subfigure}%
\caption{Distribution of sample points by ASADG (red points) and LHS (blue points) of (a) input data and (b) manifold points in the input data domain.} \label{fig_method_dim_high}
\end{figure}

\subsubsection{Data generation time}
To evaluate the generation time, the computational time for solving the equation and for computing the residual are estimated. The average time of a sample for solving and computing the residual are of order $0.00105s\ \textrm{and}\ 0.00036s$ respectively. Table\@ \ref{tab_gen_time} depicts the number of rejected points and of sample points to be generated for the low and high dimension data. The number of rejected points is almost the same as that of the generated points. However, the computational time for the rejected points is almost 2.5 times less than that of the generated points for both cases. If we consider the same amount of time for the LHS or uniform sampler, the number of generated points is of $700$ sample points. In the following study, the number generated training data of $500$ for both cases is considered.

\begin{table}[ht]
\caption{Generation time}\label{tab_gen_time}
\centering
\begin{tabular}{|l|c|c|}
\hline
 &Number of computation&Computational time (s)\\
 \hline
 \hline
Low dimension&&\\
 \hline
Solver& 500 & 0.525\\
\hline
Residual& 563 & 0.202\\
\hline
Total time (s)&& 0.727\\
\hline
 \hline
High dimension&&\\
\hline
Solver& 500 & 0.525\\
\hline
Residual& 611 & 0.220\\
\hline
Total time (s)&& 0.745\\
\hline
\end{tabular}
\end{table}

\subsection{Benchmarks on neural network performance}
This section evaluates the performance a Fully Connected Neural Network (FCNN) using ASADG for low dimensional input data where LHS is used, and high dimensional input data where the uniform sampler is utilized. When ASADG is applied for sampling input data, the number of samples is determined by the iteration number where the metric average value is converged to a certain value. The number of training sampling points of the low and high dimensional input data are both of 500 as described in section \ref{subsec:compsamplers}. Moreover, the test data set is generated by the uniform sampler and its number is of $10\%$ of that of the training data. Table\@ \ref{tab_fcnn} depicts the hyper-parameters of the neural networks employed for low and high dimenstion dataset.

To evaluate the solution, the mean normalized relative error (MNRE) metric is chosen. The metric is defined as follows:
\[
d(\hat y, y) = \frac{1}{M}\sum\limits_{i\in\{1...M\}}\frac{|\hat u_i - u_i|}{|u_i|}
\]

where 
\begin{itemize}
\item $y$ is the targeted solution ;
\item $\hat y$ is the predicted solution ;
\item $\hat u_i, u_i$ is the normalized components of the predicted and targeted solution and are defined as below:
\end{itemize}
\[
u_i = \epsilon + \frac{y_i - y_{i,min}}{y_{i,max}-y_{i,min}}
\]
\begin{itemize}
\item $y_{i,min},y_{i,max}$ are minimum and maximum value of each component of the solutions in the test sample set ;
\item $\epsilon\in]0,1[$ is a constant such that the normalized component $u_i\in[\epsilon,1+\epsilon]$ ;
\item $\epsilon=0.1$ for the following study.
\end{itemize}
Table\@ \ref{tab_fcnn_prec} shows that the MNRE of ASADG sampler improves $33\%$ better in term of neural network precision performance comparing to LHS sampler for the low dimensional input data. As for the high dimensional input data, ASADG sampler improves nearly $16\%$ better for MNRE in term of neural network precision performance comparing to uniform sampler. Moreover, in terms of standard deviation of the errors, for both cases, it is two time less spread than that of either LHS or uniform sampler.

\begin{table}[h]
\caption{Fully connected neural network hyper-parameters}\label{tab_fcnn}
\centering
\begin{tabular}{|l|c|c|}
\hline
&Low dimension&High dimension\\
\hline
\hline
Input dimension & 2 &3072\\
\hline
Output dimension & 2048 &2048\\
\hline
$\#$Layer & 2 &2\\
\hline
$\#$Layer's Node& $\{64,512\}$ &$\{3000,2500\}$\\
\hline
Activation function & Tanh &Tanh\\
\hline
Loss function& MAE &MAE\\
\hline
Optimizer & SGD &SGD\\
\hline
Learning rate & 0.01 &0.01\\
\hline
Epoch & 300  &300\\
\hline
Training samples & 500&500\\
\hline
Batch sizes & 100  &100\\
\hline
\end{tabular}
\end{table}

\begin{table}[ht]
\caption{Fully connected neural network precision performance for second case}\label{tab_fcnn_prec}
\centering
\begin{tabular}{|l|c|c|}
\hline
Low dimension&&\\
\hline
Error type/Sampler &LHS&ASADG\\
\hline
MNRE & 8.098 & 5.970\\
\hline
Standard deviation & 409.007 & 228.511\\
\hline
\hline
High dimension&&\\
\hline
Error type/Sampler &Uniform&ASADG\\
\hline
MNRE & 1.819 & 1.552\\
\hline
Standard deviation&60.367&33.571\\
\hline
\end{tabular}
\end{table}

\section{Summary}

We developed an adaptive sampling algorithm for data generation (ASADG) relying on a physical model. This algorithm allows to build a high dimensional data manifold using triangulation by adding new manifold points step by step with respect to a metric threshold. For high dimensional data, we obtained the reduced surface manifold. For numerical experiments, harmonic transport problem is employed to create two cases for low and high dimensional parametric space. Then, we studied the sensibility of the ASADG's parameters for the low dimension case. After that, the comparison of ASADG with LHS and uniform sampler for low dimension and high dimension respectively is carried out. The results showed a strong interest of using ASADG to capture the data manifold feature; nevertheless, high number of rejected points are generated which implies that additional computational time is needed to compute the equation residual. Finally, benchmarks on the FCNN's precision performance with training sampling created by ASADG for low and high dimensional data comparing to LHS and uniform sampler respectively are carried out. Once again, ASADG sampler has shown a better precision performance than one shot samplers in both cases. However, the comparisons are performed only on one shot samplers for a single experiment on both cases. Further works on benchmarking this method with different sequential samplers need to be investigated. 

\section*{Acknowledgements}
The authors also thank the industrial partners for their supports in the framework of the HSA project \cite{irtsysx}. This project is as well supported by the French government's aid in the framework of PIA (Programme d'Investissement d'Avenir) for Institut de Recherche Technologique SystemX.

\bibliographystyle{unsrt} 
\bibliography{references}

\begin{thebibliography}{10}

\bibitem{dupuis2019surrogate}
Romain Dupuis.
\newblock {\em Surrogate models coupled with machine learning to approximate
  complex physical phenomena involving aerodynamic and aerothermal
  simulations}.
\newblock PhD thesis, Institut National Polytechnique de Toulouse-INPT, 2019.

\bibitem{guenot2013adaptive}
Marc Gu{\'e}not, Ingrid Lepot, Caroline Sainvitu, Jordan Goblet, and Rajan
  Filomeno~Coelho.
\newblock Adaptive sampling strategies for non-intrusive pod-based surrogates.
\newblock {\em Engineering computations}, 30(4):521--547, 2013.

\bibitem{mandelli2012adaptive}
D~Mandelli and C~Smith.
\newblock Adaptive sampling using support vector machines.
\newblock Technical report, Idaho National Lab.(INL), Idaho Falls, ID (United
  States), 2012.

\bibitem{zhou2015adaptive}
Qi~Zhou, Xinyu Shao, Ping Jiang, Hui Zhou, and Leshi Shu.
\newblock An adaptive global variable fidelity metamodeling strategy using a
  support vector regression based scaling function.
\newblock {\em Simulation Modelling Practice and Theory}, 59:18--35, 2015.

\bibitem{fuhg2021state}
Jan~N Fuhg, Am{\'e}lie Fau, and Udo Nackenhorst.
\newblock State-of-the-art and comparative review of adaptive sampling methods
  for kriging.
\newblock {\em Archives of Computational Methods in Engineering},
  28:2689--2747, 2021.

\bibitem{liu2018survey}
Haitao Liu, Yew-Soon Ong, and Jianfei Cai.
\newblock A survey of adaptive sampling for global metamodeling in support of
  simulation-based complex engineering design.
\newblock {\em Structural and Multidisciplinary Optimization}, 57:393--416,
  2018.

\bibitem{ajdari2014adaptive}
Ali Ajdari and Hashem Mahlooji.
\newblock An adaptive exploration-exploitation algorithm for constructing
  metamodels in random simulation using a novel sequential experimental design.
\newblock {\em Communications in Statistics-Simulation and Computation},
  43(5):947--968, 2014.

\bibitem{jiang2018adaptive}
Ping Jiang, Yahui Zhang, Qi~Zhou, Xinyu Shao, Jiexiang Hu, and Leshi Shu.
\newblock An adaptive sampling strategy for kriging metamodel based on delaunay
  triangulation and topsis.
\newblock {\em Applied Intelligence}, 48:1644--1656, 2018.

\bibitem{lashof1969triangulation}
R~Lashof and M~Rothenberg.
\newblock Triangulation of manifolds. i.
\newblock {\em Bulletin of the American Mathematical Society}, 75(4):750--754,
  1969.

\bibitem{whitehead1940c1}
John Henry~C Whitehead.
\newblock On c1-complexes.
\newblock {\em Annals of Mathematics}, pages 809--824, 1940.

\bibitem{moise1952affine}
Edwin~E Moise.
\newblock Affine structures in 3-manifolds: V. the triangulation theorem and
  hauptvermutung.
\newblock {\em Annals of mathematics}, pages 96--114, 1952.

\bibitem{rado1925begriff}
Tibor Rad{\'o}.
\newblock {\"U}ber den begriff der riemannschen fl{\"a}che.
\newblock {\em Acta Litt. Sci. Szeged}, 2(101-121):10, 1925.

\bibitem{freedman1982topology}
Michael~Hartley Freedman.
\newblock The topology of four-dimensional manifolds.
\newblock {\em Journal of Differential Geometry}, 17(3):357--453, 1982.

\bibitem{kirby1969triangulation}
Robion~C Kirby, Laurence~C Siebenmann, et~al.
\newblock On the triangulation of manifolds and the hauptvermutung.
\newblock {\em Bull. Amer. Math. Soc}, 75(4):742--749, 1969.

\bibitem{boissonnat2018delaunay}
Jean-Daniel Boissonnat, Ramsay Dyer, and Arijit Ghosh.
\newblock Delaunay triangulation of manifolds.
\newblock {\em Foundations of Computational Mathematics}, 18:399--431, 2018.

\bibitem{van2008visualizing}
Laurens van~der Maaten and Geoffrey Hinton.
\newblock Visualizing data using t-sne. journal of machine learning research 9.
\newblock {\em Nov (2008)}, 2008.

\bibitem{roweis2000nonlinear}
Sam~T Roweis and Lawrence~K Saul.
\newblock Nonlinear dimensionality reduction by locally linear embedding.
\newblock {\em science}, 290(5500):2323--2326, 2000.

\bibitem{leland2018uniform}
McInnes Leland, Healy John, and Melville James.
\newblock Uniform manifold approximation and projection for dimension
  reduction.
\newblock {\em arXiv preprint arXiv:1802.03426}, 2018.

\bibitem{bensalah2022mathematical}
Antoine Bensalah, Patrick Joly, and Jean-Francois Mercier.
\newblock Mathematical analysis of goldstein’s model for time-harmonic
  acoustics in flows.
\newblock {\em ESAIM: Mathematical Modelling and Numerical Analysis},
  56(2):451--483, 2022.

\end{thebibliography}

\appendix

\section{Proof of theorems}\label{sec:theorem}
\begin{proof}[Proof of Lemma \ref{lemma:proj}]
We will proove the existence of $\pi = (\mu,\nu)$ by constructing the mappings $\mu$ and $\nu$ corresponding to the input and output data respectively and their pseudo inverse $\mu^{-1}$ and $\nu^{-1}$. 

Assume that $N\geq 2$. Hence, the projector $\mu$ is defined by the use of the UMAP projector such that for $p\in \mathbf{R}^N, \mu(p)=\bar p \in \mathbf R^2$ and its pseudo inverse $\mu^{-1}$ is constructed according to the inverse projector mentioned in the section \ref{sec:reducedmanifold}, hence $||\mu^{-1}(\bar{p}) - p||_{L^2}<\epsilon$. 

Let's define $\nu : \mathbf{R}^M \rightarrow \mathbf{R}$ and $y\rightarrow ||y||_{L_2}=\bar{y}$. Therefore, we can construct a direct mapping $\pi:\mathbf{R}^N\times\mathbf{R}^M\rightarrow \mathbf{R}^2\times\mathbf{R}$ and $\pi(P) = (\mu(p),\nu(y))=\bar{P}$.

Let's construct $\nu^{-1}:\mathbf{R}\rightarrow \mathbf{R}^M$ and $\nu^{-1}(\bar{z}) = \tilde{y}$ such that $\tilde{P}=(\tilde p, \tilde y)\in \mathcal{M}$,  thus by equation (\ref{eq_manifold}) with tolerance constant $\epsilon$ : 
\[
\chi (F(\tilde P) - \epsilon) = 0\ \textrm{or}\ F(\tilde p,\tilde y) \leq \epsilon
\]
Thus for $\mu^{-1}(\bar{p}) = \tilde{p}$ and $||\mu^{-1}(\bar{p}) - p||_{L^2}<\epsilon$ (by the construction of pseudo inverse mapping), then $||\tilde{p} - p||_{L^2}<\epsilon$. 

By continuity of the manifold $\mathcal{M}$ or implicit function $F$, hence $||\tilde y - y||_{L^2}<\epsilon$ . 

Therefore, $||\tilde{P} - P||_{L^2}^2=||\tilde{p} - p||_{L^2}^2+||\tilde{y} - y||_{L^2}^2<2\epsilon^2$, thus $||\tilde{P} - P||_{L^2}<\sqrt{2}\epsilon$. 

We define the projector 
\[\pi^{-1}:\mathbf{R}^2\times\mathbf{R}\rightarrow \mathbf{R}^N\times\mathbf{R}^M\] 
and 
\[\pi^{-1}(\bar{P}) = (\mu^{-1}(\bar{p}),\nu^{-1}(\bar{z}))=\tilde{P}\] then $||\pi^{-1}(\bar{P})-P||_{L_2}<\sqrt{2}\epsilon$.

Therefore, $\pi^{-1}$ is the pseudo inverse projector of $\pi$ 
\end{proof}

\begin{proof}[Proof of Theorem \ref{thm:reducedmanifold}]
From Lemma \ref{lm_pi}:
\[||P - \pi^{-1}(\pi(P))||_{L^2}<\epsilon\]
Thus,
\[||P - \pi^{-1}(\bar{P})||_{L^2}<\epsilon\]

Chosen $\epsilon$ sufficiently small as tolerance constant of $F$ at point $P\in\mathcal P$ and for the continuous $\mathcal{M}$, then:
\[|F(\pi^{-1}(\bar{P}))|< |F(P)| + |F(P) -F(\pi^{-1}(\bar{P}))|  < 2\epsilon \]

Let's define $\bar{F} = F \circ \pi^{-1}$ then for $\bar P\in \bar{\mathcal P}, \bar F(\bar P) \leq 2\epsilon$

$2\epsilon$ can be defined as tolerance constant of $\bar{F}$. Therefore, the reduced order manifold can be constructed as follows:
\[
\bar{\mathcal{M}}=\{\bar{P}\in \bar{\mathbf{P}} | \chi(\bar{F}(\bar{P}) - 2\epsilon)= 0\}
\] 
\end{proof}

\begin{proof}[Proof of corollary \ref{cor:reducedmanifold}]
The same proof as in theorem \ref{th_manifold_reduced}.
\end{proof}

\end{document}